\documentclass[a4paper]{llncs}

\usepackage{amssymb}
\usepackage{graphicx}
\usepackage{subfigure} 
\usepackage{xcolor}

\usepackage[numbers]{natbib}

\usepackage{algorithm}
\usepackage{algorithmic}
\usepackage{array}

\usepackage{enumerate}

\usepackage{hyperref}

\usepackage{amsmath, amssymb, amsfonts}

\usepackage{lscape}

\usepackage{url}
\urldef{\mailsa}\path|aharutyu@vub.ac.be|
\urldef{\mailsb}\path|{bellemare, stepleton, munos}@google.com|

\newcommand{\beq}{\begin{equation}}
\newcommand{\eeq}{\end{equation}}

\newcommand{\beqa}{\begin{eqnarray}}
\newcommand{\eeqa}{\end{eqnarray}}

\newcommand{\beqan}{\begin{eqnarray*}}
\newcommand{\eeqan}{\end{eqnarray*}}

\renewcommand{\P}{\mathbb{P}}
\renewcommand{\Pr}{\mathbb{P}}

\newcommand{\E}{\mathbb{E}}

\newcommand{\indic}[1]{\mathbb{I}\{#1\}}

\let\R\undefined 
\newcommand{\Real}{\mathbb{R}}

\newcommand{\bN}{\mathbb{N}}
\newcommand{\cbar}{\,|\,}

\DeclareMathOperator*{\argmin}{arg\,min}
\DeclareMathOperator*{\argmax}{arg\,max}

\renewcommand{\phi}{\varphi}
\renewcommand{\epsilon}{\varepsilon}

\newcommand{\A}{\mathcal{A}}

\newcommand{\X}{\mathcal{X}}

\newcommand{\T}{\mathcal{T}}

\newcommand{\R}{\mathcal{R}}

\newcommand{\Rmax}{R_{\textsc{max}}}

\newcommand\defequal{\mathrel{\overset{\makebox[0pt]{\mbox{\tiny
          def}}}{=}}}

\newcommand{\algrule}[1][.2pt]{\par\vskip.5\baselineskip\hrule height #1\par\vskip.5\baselineskip}

\begin{document}

\mainmatter  

\title{Q($\lambda$) with Off-Policy Corrections}


\author{Anonymous}

\author{Anna Harutyunyan$^{1}\thanks{This work was carried out during an internship at Google DeepMind.}$
\and Marc G. Bellemare$^2$ \and Tom Stepleton$^2$\and R\'{e}mi Munos$^2$}


\institute{
$^1$ VU Brussel \\
$^2$ Google DeepMind\\
\mailsa\\
\mailsb}

\maketitle

\begin{abstract}
We propose and analyze an alternate approach to off-policy multi-step
temporal difference learning, in which off-policy returns are
corrected with the current Q-function in terms of rewards, rather than
with the target policy in terms of transition
probabilities. We prove that such approximate
corrections are sufficient for off-policy convergence both in policy
evaluation and control, provided certain conditions. These conditions
relate the distance between the target and behavior policies, the eligibility
trace parameter and the discount factor, and formalize an underlying
tradeoff in off-policy TD($\lambda$). 
We illustrate this theoretical relationship 
empirically on a
continuous-state control task.
\end{abstract}

\section{Introduction}

In reinforcement learning (RL), learning is off-policy when samples 
generated by a {\em behavior} policy are used to learn about a distinct {\em
  target} policy. The usual approach to off-policy learning
 is to disregard, or altogether discard transitions whose target policy probabilities
 are low. 
For example, Watkins's Q($\lambda$)~\cite{watkins1992}
cuts the trajectory backup as soon as a non-greedy action is
encountered. Similarly, in policy evaluation, importance sampling
methods~\cite{precup2000eligibility} weight the returns according to
the mismatch in the target and behavior probabilities of the
corresponding actions. This approach treats transitions
conservatively, and hence may unnecessarily terminate backups, 
or introduce a large amount of variance.  

Many off-policy methods, in particular of the Monte Carlo kind, have
no other option than to judge off-policy actions in the probability sense.
However, {\em temporal difference} methods~\cite{sutton1988learning} in RL maintain an
approximation of the value function along the way, with {\em
  eligiblity traces}~\cite{watkins1989learning} providing a continuous link 
between one-step and Monte Carlo approaches. 
  The value function assesses actions in terms of the following expected
cumulative reward, and thus provides a way to directly correct
immediate {\em rewards}, rather than transitions.
We show in this paper that such approximate corrections can be
sufficient for off-policy convergence, subject to a tradeoff condition
between the eligibility
trace parameter and the distance between the target and behavior policies. The
two extremes of this tradeoff are one-step Q-learning, and on-policy
learning. Formalizing the continuum of the tradeoff is one of the
main insights of this paper.

In particular, we propose an off-policy return operator that 
augments the return with a correction term, based on the current
approximation of the Q-function. We then formalize three algorithms
stemming from this operator: (1) off-policy
Q$^\pi$($\lambda$), and its special case (2) {\em on}-policy
Q$^\pi$($\lambda$), for policy evaluation, and (3) Q$^*$($\lambda$) for
off-policy control. 

In policy evaluation, both on- and off-policy Q$^\pi$($\lambda$) are
novel, but closely related to several existing algorithms of the
TD($\lambda$) family. Section~\ref{sec:related_work}
discusses this in detail.  We prove convergence of Q$^\pi$($\lambda$),
subject to the $\lambda-\epsilon$ tradeoff where $\epsilon\defequal\max_x \|\pi(\cdot|x) - \mu(\cdot|x)\|_1$ 
is a measure of dissimilarity between the behavior and target policies.
More precisely, we prove that for any amount of
``off-policy-ness'' $\epsilon\in[0,2]$ there is an inherent maximum allowed backup length
value $\lambda= \frac{1-\gamma}{\gamma \epsilon}$, and taking $\lambda$ below this value guarantees convergence
to $Q^{\pi}$ without involving policy probabilities. This is desirable due to the
instabilities and variance introduced by the likelihood ratio products in the importance sampling
approach \cite{Precup2001}.

In control, Q$^*$($\lambda$) is in fact
identical to Watkins's Q($\lambda$), except
it does not cut the eligiblity trace at off-policy
actions. \citet{sutton-barto98} mention such a variation, which they call
{\em naive} Q($\lambda$). We analyze this algorithm for the first
time and prove its convergence for small values of $\lambda$. Although we
were not able to prove a $\lambda-\epsilon$ tradeoff similar to the policy evaluation
case, we provide empirical evidence for the existence of such a tradeoff, confirming the intuition that
naive Q($\lambda$) is ``not as naive as one might at first suppose''~\cite{sutton-barto98}. 

We first give the technical background, and define our operators. We
then specify the incremental versions of our algorithms based on these
operators, and
state their convergence. We follow by proving convergence:
subject to the $\lambda-\epsilon$ tradeoff in policy evaluation, and
more conservatively, for small values of $\lambda$ in control. 
We illustrate the tradeoff emerge empirically in the Bicycle
domain in the control setting. Finally, we conclude by placing our algorithms in context within
existing work in TD($\lambda$).

\section{Preliminaries}
 \label{sec:preliminaries}

We consider an environment modelled by the usual discrete-time Markov Decision
Process $(\X, \A, \gamma, P, r)$ composed of the finite
state and action spaces $\X$ and $\A$, a discount factor $\gamma$, a
transition function $P$ mapping each $(x,a) \in (\X, \A)$ to a
distribution over $\X$, and a reward function $r : \X \times \A \to
[-\Rmax, \Rmax]$. A \emph{policy} $\pi$ maps a state $x \in \X$ to a distribution over $\A$.
A Q-function $Q$ is a mapping $\X \times \A \to \Real$. Given a policy $\pi$, we define the operator $P^\pi$ over Q-functions:
\begin{equation*}
(P^\pi Q)(x,a) \defequal \sum_{x' \in \X} \sum_{a' \in \A} P(x' \cbar
x, a) \pi(a' \cbar x') Q(x', a') .
\end{equation*}
To each policy $\pi$ corresponds a unique Q-function $Q^\pi$ which describes the
expected discounted sum of rewards achieved when following $\pi$:
\begin{equation}
  Q^\pi \defequal \sum_{t \ge 0} \gamma^t (P^\pi)^t r, \label{eq:Qpi_def}
\end{equation}
where for any operator $X$, $(X)^t$ denotes $t$ successive applications of
$X$, and where we commonly treat $r$ as one particular Q-function. We write the {\em
  Bellman operator} $\T^\pi$, and the {\em Bellman equation} for $Q^\pi$:
 \begin{align}
 \T^\pi Q & \defequal r + \gamma P^\pi Q, \notag \\
  \T^\pi Q^\pi & = Q^\pi = (I - \gamma P^\pi)^{-1} r. \label{eq:Bellman_Qpi}
 \end{align}
The  {\em Bellman optimality operator} $\T$ is defined as
$  \T Q  \defequal r + \gamma \max_\pi P^{\pi} Q, $
and it is well known~\citep[e.g.][]{Bellman:1957, Puterman:1994} that the optimal Q-function $Q^*\defequal
\sup_\pi Q^\pi$ is the unique solution to the Bellman optimality equation 
\beq  \label{eq:T*}
\T Q = Q.
\eeq
We write $\textsc{Greedy}(Q) \defequal \{\pi | \pi(a|x) > 0
\Rightarrow Q(x,a) = \max_{a'}Q(x,a')\}$ to denote the set of greedy policies w.r.t.~$Q$. Thus $\T Q = \T^{\pi}Q$ for any $\pi\in \textsc{Greedy}(Q)$.

Temporal difference (TD) learning~\cite{sutton1988learning} rests on the fact that 
iterates of both operators $\T^\pi$ and $\T$ are guaranteed to converge to their
respective fixed points $Q^\pi$ and $Q^*$
. Given a sample experience
$x,a,r,x',a'$, SARSA(0)~\cite{rummery1994line} updates its Q-function
estimate at $k^{th}$ iteration as follows:
\begin{align*}
  Q_{k+1}(x,a) & \gets Q_k(x,a) + \alpha_k\delta, \\
  \delta &= r + \gamma Q_k (x', a') - Q_k(x,a),
\end{align*}
where $\delta$ is the {\em TD-error}, and $(\alpha_k)_{k\in\bN}$ a sequence of nonnegative
stepsizes.
One need not only consider short experiences, but may sample
trajectories $x_0,a_0,r_0,x_1,a_1,r_1,\ldots$, 
and accordingly apply
$\T^\pi$ (or $\T$) repeatedly. A particularly flexible way of doing
this is via a weighted sum $A^\lambda$ of such {\em $n$-step} operators: 
\begin{align*}
\T_\lambda^\pi Q &\defequal A^\lambda [ (\T^\pi)^{n + 1} Q ] \\
& = Q + (I - \lambda\gamma P^\pi)^{-1} (\T^\pi Q - Q), \\
A^\lambda [ f(n)] &\defequal (1 - \lambda) \sum_{n \ge 0} \lambda^n f(n).
\end{align*}

Naturally, $Q^\pi$ remains the fixed point of $\T_\lambda^\pi$. Taking
$\lambda = 0$ yields the usual Bellman operator $\T^\pi$, and $\lambda = 1$ removes the recursion on the approximate
Q-function, and restores $Q^\pi$ in the {\em Monte
Carlo} sense. It is well-known that $\lambda$ trades off the bias
from {\em bootstrapping} with an approximate Q-function, with the
variance from using a sampled multi-step return~\cite{kearns2000bias},
with intermediate values of $\lambda$ usually performing best in
practice~\citep{sutton1996generalization,singh1998analytical}. The above $\lambda$-operator can be efficiently implemented in the online setting via
a mechanism called {\em eligibility traces}.
 As we will see in Section \ref{sec:related_work}, it in
 fact corresponds to a number of online algorithms, each subtly
 different, of which SARSA($\lambda$)~\cite{rummery1994line} is the canonical instance.

Finally, we make an important distinction between the {\em target policy $\pi$}, which we wish to estimate, and
the {\em behavior policy
$\mu$}, from which the actions have been generated. If
$\mu=\pi$, the learning is said to be {\em  on-policy}, otherwise it is {\em off-policy}.
We will write $\E_\mu$ to denote expectations over sequences
$x_0,a_0,r_0,x_1,a_1,r_1,\ldots$, $a_i\sim\mu(\cdot|x_i)$,
$x_{i+1}\sim P(\cdot | x_i, a_i)$ and assume conditioning on
$x_0 = x$ and $a_0 = a$ wherever appropriate. Throughout, we will write
$\|\cdot\|$ for supremum norm. 

\section{Off-Policy Return Operators}
\label{sec:return-operators}

We will now describe the Monte Carlo {\em off-policy corrected return operator}
$\R^{\pi,\mu}$ that is at the heart of our contribution. Given a
target $\pi$, and a return generated by the behavior $\mu$, the
operator $\R^{\pi,\mu}$ 
attempts to approximate a return that would have been
generated by $\pi$, by utilizing a correction built from a current
approximation $Q$ of $Q^{\pi}$. Its application to $Q$ at a state-action pair $(x,a)$ is defined as follows:
\begin{align}
    (\R^{\pi,\mu} Q)(x,a) 
   \defequal r(x,a) + \E_\mu\big[
   \sum_{t\geq 1}\gamma^t \big(r_t+
   \underbrace{\E_\pi{Q(x_t,\cdot)} - Q(x_t,a_t)}_{\mbox{off-policy correction}}\big) \big],\label{eq:Rxa}
  \end{align}
where 
we use the shorthand $\E_\pi Q(x, \cdot) \equiv
\sum_{a\in\A}\pi(a|x) Q(x,a)$.

That is, $\R^{\pi,\mu}$ gives the usual expected discounted sum of future rewards, but each
reward in the trajectory is augmented with an {\em off-policy
  correction}, which we define as the difference between the {\em expected}
(with respect to the target policy) Q-value and the Q-value for 
  the taken action. 
Thus, how much a reward is corrected is determined by
both the approximation $Q$, and the target policy probabilities. 
Notice that if actions are similarly valued,
the correction will have little effect, and
learning will be roughly on-policy, but if the Q-function has converged to
the correct estimates $Q^\pi$, the correction takes the
immediate reward $r_t$ to the expected reward with respect to $\pi$
exactly. Indeed, as we will see later, $Q^\pi$ is the fixed point of $\R^{\pi,\mu}$ for any behavior policy $\mu$.

We define the $n$-step and $\lambda$-versions of $\R^{\pi,\mu}$ in the
usual way:
\beqa
   \R^{\pi,\mu}_\lambda Q &\defequal& A^\lambda [ \R^{\pi,\mu}_n ], \label{eq:lambdaR} \\
    (\R^{\pi,\mu}_n Q)(x,a)& \defequal& r(x,a) +
    \E_\mu\big[\sum_{t=1}^n\gamma^t \big(r_t+  \E_\pi{Q(x_t,\cdot)} -
    Q(x_t,a_t)\big) \notag 
   \\  & & 
+ \gamma^{n+1}\E_\pi
     Q(x_{n+1},\cdot)\big]. \notag \label{eq:Rnxa}
\eeqa
Note that the $\lambda$ parameter here takes us from TD($0$)
to the Monte Carlo version of our operator $\R^{\pi,\mu}$, rather than
the traditional Monte Carlo form~\eqref{eq:Qpi_def}.

\begin{algorithm}\label{alg:off_policy_online_algorithm}
\caption{Q($\lambda$) with off-policy corrections}
\begin{algorithmic}
\medskip
\item[\textbf{Given:}] Initial $Q$-function $Q_0$, stepsizes $(\alpha_k)_{k \in \bN}$
\FOR{$k = 1 \dots$}
    \STATE Sample a trajectory $x_0, a_0, r_0, \dots, x_{T_k}$ from $\mu_k$
    \STATE $Q_{k+1}(x,a) \gets Q_k(x,a) \qquad \forall x, a$
    \STATE $e(x,a) \gets 0 \qquad \forall x, a$ 
    \FOR{$t = 0 \dots T_k-1$}
        \STATE $\delta^{\pi_k}_t \gets r_t + \gamma \E_{\pi_k} Q_{k+1}(x_{t+1}, \cdot) - Q_{k+1}(x_t, a_t)$ 
        \FORALL{$x \in \X, a \in \A$}
            \STATE $e(x,a) \gets \lambda \gamma e(x,a) + \indic{(x_t, a_t) = (x, a)}$ 
            \STATE $Q_{k+1}(x,a) \gets Q_{k+1}(x,a) + \alpha_k \delta^{\pi_k}_t e(x,a)$ 
        \ENDFOR
    \ENDFOR
\ENDFOR
\\\algrule[1pt]
\item[\textbf{On-policy Q$^\pi$($\lambda$):}] $\mu_k = \pi_k = \pi$.
\item[\textbf{Off-policy Q$^\pi$($\lambda$):}] $\mu_k \ne \pi_k = \pi$.
\item[\textbf{Q$^*$($\lambda$):}] $\pi_k \in \textsc{Greedy}(Q_k)$.
\end{algorithmic}
\end{algorithm}

\section{Algorithm}\label{sec:algorithm}

We consider the problems of \emph{off-policy policy evaluation} and
\emph{off-policy control}. In both problems we are given
data generated by a sequence of behavior
policies $( \mu_k)_{k \in \bN}$. In policy evaluation, we wish to
estimate $Q^{\pi}$ for a fixed target policy $\pi$. 
In
control, we wish to estimate $Q^*$. Our algorithm constructs a sequence $(Q_k)_{k \in \bN}$ of estimates of
$Q^{\pi_k}$ from trajectories sampled from $\mu_k$, by applying the
$\R^{\pi_k,\mu_k}_\lambda$-operator:
\begin{equation}
  \label{eq:RQ_alg}
  Q_{k+1} = \R^{\pi_k,\mu_k}_\lambda Q_k,
\end{equation}
where $\pi_k$ is the $k^{th}$ interim target policy. We distinguish between three algorithms:
\begin{description}
\item[Off-policy Q$^\pi$($\lambda$) for policy evaluation:] $\pi_k =
  \pi$ is the fixed target
  policy. We write the corresponding operator $\R^\pi_\lambda$.
\item [On-policy Q$^\pi$($\lambda$) for policy evaluation:] 
for the special case of $\mu_k=\mu=\pi$.
\item[Q$^*$($\lambda$) for off-policy control:] $(\pi_k)_{k \in \bN}$ is a sequence
of greedy policies with respect to 
$Q_k$.
We write the corresponding operator
$\R^*_\lambda$.
\end{description}
We wish to write the update \eqref{eq:RQ_alg} in terms of a simulated
trajectory $x_0,a_0,r_0,\dots,$ $x_{T_k}$ drawn according to $\mu_k$.
First, notice that \eqref{eq:lambdaR} can be rewritten:
\begin{align*}
    \R^{\pi,\mu}_\lambda Q(x,a) & = Q(x,a) + \E_\mu\big[\sum_{t\geq
      0}(\lambda\gamma)^t\delta^\pi_t\big], \\
 \delta^\pi_t & \defequal r_t+ \gamma\E_\pi{Q(x_{t+1},\cdot)} - Q(x_t,a_t),
\end{align*}
where $\delta^\pi_t$ is the {\em expected} TD-error.
The {\em offline} forward view\footnote{The true online version can be derived as given by
  \citet{seijen2014true}} is then
\begin{align}
Q_{k+1}(x,a) &\gets Q_k(x,a) + \alpha_k\sum_{t=0}^{T_k} (\gamma \lambda)^t
\delta^{\pi_k}_t, \label{eqn:online_q_pi_update} 
\end{align}
While \eqref{eqn:online_q_pi_update} resembles 
many existing TD($\lambda$)
algorithms, it subtly differs from all of them, due to
$\R^{\pi,\mu}_\lambda$ (rather than $\T^\pi_\lambda$) being at its
basis. Section~\ref{sec:related_work} discusses the distinctions in
detail. The practical {\em every-visit}~\cite{sutton-barto98} form of
\eqref{eqn:online_q_pi_update} is written
\begin{align}
  Q_{k+1}(x,a) &\gets Q_k(x,a) + \alpha_k\sum_{t=0}^T
                 \delta^{\pi_k}_t\sum_{s = 0}^t (\gamma
                 \lambda)^{t-s} \indic{(x_s, a_s) = (x, a)},
                 \label{eqn:online_q_pi_update_every_visit}
\end{align}
and the corresponding online backward view of all three
algorithms is summarized in Algorithm 1.

The following theorem states that when $\mu$ and $\pi$ are
sufficiently close, the off-policy Q$^\pi$($\lambda$) algorithm
converges to its fixed point $Q^\pi$.

\begin{theorem}\label{thm:online_convergence}
Consider the sequence of Q-functions computed according to Algorithm 1
with fixed policies $\mu$ and $\pi$. Let $\epsilon = \max_x \|\pi(\cdot|x) - \mu(\cdot|x)\|_1$. If $\lambda \epsilon < \frac{1-\gamma}{\gamma}$, then under the
same conditions required for the convergence of $TD(\lambda)$ (1--3 in
Section 5.3)
we have,
almost surely:
\begin{equation*}
\lim_{k\to \infty} Q_k(x,a) = Q^\pi(x,a) .
\end{equation*}
\end{theorem}
 We state a similar, albeit weaker result for
 Q$^*$($\lambda$). 
 \begin{theorem}\label{thm:online_convergence_q*}
Consider the sequence of Q-functions computed according to Algorithm 1
with $\pi_k$ the greedy policy with respect to $Q_k$. If $\lambda <
\frac{1-\gamma}{2\gamma}$, then under the same conditions required for the convergence of
TD($\lambda$) (1--3 in
Section 5.3)
we have, almost surely:
\begin{equation*}
\lim_{k \to \infty} Q_k(x,a) = Q^*(x,a) .
\end{equation*}
\end{theorem}


The proofs of these theorems rely on showing that
$\R^{\pi}_\lambda$ and $\R^{*}_\lambda$ are contractions (under
the stated conditions), and invoking classical stochastic
approximation convergence to their fixed point (such as Proposition
4.5 from~\cite{bertsekas1996neurodynamic}). We will
focus on the contraction lemmas, which are the crux of the proofs,
then outline the sketch of the online convergence argument.

\subsubsection{Discussion}
\label{sec:discussion}
Theorem \ref{thm:online_convergence} states that for {\em any}
$\lambda\in[0,1]$ there exists some degree of ``off-policy-ness''
$\epsilon<\frac{1-\gamma}{\lambda\gamma}$ under which $Q_k$ converges
to $Q^{\pi}$. This is the $\lambda-\epsilon$ tradeoff for the
off-policy $Q^{\pi}(\lambda)$ learning algorithm for policy
evaluation. In the control case, the result of
Theorem~\ref{thm:online_convergence_q*} is weaker as it only holds for values of $\lambda$ smaller than $\frac{1-\gamma}{2\gamma}$. Notice that this threshold corresponds to the policy evaluation case for $\epsilon=2$ (arbitrary off-policy-ness). We were not able to prove convergence to $Q^*$ for any $\lambda\in[0,1]$ and some $\epsilon>0$. This is left as an open problem for now. 

The main technical difficulty lies in the fact that in control, the
greedy policy with respect to the current $Q_k$ may change drastically
from one step to the next, while $Q_k$ itself changes incrementally
(under small learning steps $\alpha_k$). So the current $Q_k$ may not offer a good off-policy correction to evaluate the new greedy policy. In order to circumvent this problem we may want to use slowly changing target policies $\pi_k$. For example we could keep $\pi_k$ fixed for slowly increasing periods of time. This can be seen as a form of optimistic policy iteration \cite{Puterman:1994} where policy improvement steps alternate with approximate policy evaluation steps (and when the policy is fixed, Theorem~\ref{thm:online_convergence} guarantees convergence to the value function of that policy). Another option would be to define $\pi_k$ as the empirical average $\pi_k\defequal\frac{1}{k}\sum_{i=1}^k \pi_i'$ of the previous greedy policies $\pi'_i$. We conjecture that defining $\pi_k$ such that (1) $\pi_k$ changes slowly with $k$, and (2) $\pi_k$ becomes increasingly greedy, then we could extend the $\lambda - \epsilon$ tradeoff of Theorem \ref{thm:online_convergence} to the control case. This is left for future work.

\section{Analysis}
\label{sec:analysis}

We begin by verifying that the fixed points of $\R^{\pi,\mu}_\lambda$ in
the policy evaluation and control settings are $Q^\pi$ and $Q^*$,
respectively. We then prove the contractive properties of these
operators: $\R^{\pi}_\lambda$ is always a contraction
and will converge to its fixed point, $\R^*_\lambda$ is a contraction
for particular choices of $\lambda$ (given in terms of $\gamma$).
The contraction coefficients depend on $\lambda$, $\gamma$,
and $\epsilon$: the distance between policies. Finally, we give a
proof sketch for online convergence of Algorithm 1.

Before we begin, it will be convenient to rewrite \eqref{eq:Rxa} for all state-action pairs:
\begin{align*}
\R^{\pi,\mu} Q & = r + \sum_{t\geq1} \gamma^t (P^{\mu})^{t-1} [
P^{\mu}r + P^{\pi} Q - P^{\mu} Q ]. 
\end{align*}
We can then write $\R^\pi_\lambda$ and
$\R^*_\lambda$ from \eqref{eq:lambdaR} as follows:
\begin{align}
  \R^{\pi}_\lambda Q & \defequal  Q + (I- \lambda \gamma P^{\mu})^{-1}
  [\T^{\pi} Q - Q ], \label{eq:RlambdaQ_Tpi} \\
  \R^{*}_\lambda Q & \defequal  Q + (I- \lambda \gamma P^{\mu})^{-1}
  [\T Q - Q ]. \label{eq:RlambdaQ_Tstar}
\end{align}
It is not surprising that the above along with the Bellman equations
\eqref{eq:Bellman_Qpi} and \eqref{eq:T*} directly yields that $Q^\pi$ and $Q^*$ are the fixed points of $\R_\lambda^{\pi}$ and
$\R^*_\lambda$:
  \begin{align}
    \R^{\pi}_\lambda Q^\pi &= Q^\pi, \notag \\
    \R^*_\lambda Q^* &= Q^*. \notag 
  \end{align}
It then remains to analyze the behavior of $\R^{\pi,\mu}_{\lambda}$ as it gets iterated.

\subsection{$\lambda$-return for policy evaluation: Q$^\pi$($\lambda$)}
\label{sec:naive-qpilambda}

We first consider the case with a fixed arbitrary policy $\pi$. For
simplicity, we take $\mu$ to be fixed as well, but the same will hold
for any sequence $(\mu_k)_{k\in\bN}$, as long as each $\mu_k$ satisfies
the condition imposed on $\mu$.

\begin{lemma}\label{lem:contraction_for_q_pi} 
  Consider the policy evaluation algorithm $Q_k =
  (\R^{\pi}_\lambda)^k Q$. Assume the behavior policy $\mu$ is
  $\epsilon$-away from the target policy $\pi$, in the sense that $\max_x \|\pi(\cdot|x) - \mu(\cdot|x)\|_1\leq \epsilon$.
Then for
  $\epsilon<\frac{1-\gamma}{\lambda\gamma}$, the sequence
  $(Q_k)_{k\geq 1}$ converges to $Q^\pi$ exponentially fast:
  $\|Q_k-Q^\pi\| = O(\eta^k)$, where $\eta = \frac{\gamma}{1-\lambda\gamma}(1-\lambda + \lambda\epsilon)<1$. 
\end{lemma}

\begin{proof}
 First notice that
 \beqan
  \| P^{\pi}-P^{\mu}\| &=& \sup_{\|Q \|\leq 1} \| (P^{\pi} -P^{\mu} )Q\| \\
  &=& \sup_{\|Q \|\leq 1} \max_{x,a} \Big| \sum_y P(y|x,a) \sum_b \left((\pi(b|y) - \mu(b|y)\right) Q(y,b) \Big|\\
  &\leq& \max_{x,a} \sum_y P(y|x,a) \sum_b | \pi(b|y) - \mu(b|y) | 
\leq \epsilon.
  \eeqan
 Let $B = (I-\lambda\gamma P^{\mu})^{-1}$ be the resolvent
 matrix. From~\eqref{eq:RlambdaQ_Tpi} we have

\begin{align*}
\R^{\pi}_\lambda Q  - Q^{\pi} & = B\big[\T^\pi Q - Q + (I -
  \lambda\gamma P^\mu)(Q - Q^\pi) \big] \\
& = B\big[r + \gamma P^\pi Q - Q^\pi - \lambda\gamma P^\mu (Q - Q^\pi) \big] \\
& = B\big[\gamma P^\pi (Q - Q^\pi) - \lambda\gamma P^\mu (Q - Q^\pi)
\big] \\
& = \gamma B\big[(1-\lambda) P^\pi + \lambda(P^\pi - P^\mu)\big](Q-Q^\pi).
\end{align*}
Taking the sup norm, since $\mu$ is $\epsilon$-away from $\pi$:
\begin{align*}
\| \R^{\pi}_\lambda Q - Q^\pi\| \leq \eta\|Q - Q^\pi\|
\end{align*}
for $\eta = \frac{\gamma}{1 -
  \lambda\gamma} (1-\lambda + \lambda\epsilon) < 1$. Thus $\|Q_k -
Q^\pi\| = O(\eta^k)$.
\end{proof}

\subsection{$\lambda$-return for control:  Q$^*$($\lambda$)}
\label{sec:naive-qlambda}

We next consider the case where the $k^{th}$ target policy $\pi_k$ is
greedy with respect to the value estimate $Q_k$. 
The
following Lemma states that is possible to select a small, but nonzero $\lambda$
and still guarantee convergence.

\begin{lemma}\label{lem:contraction_for_q_star} 
Consider the off-policy control algorithm $Q_{k}=(\R_\lambda^*)^k Q$. Then 
  \begin{equation*}
  \| \R^*_\lambda Q_k - Q^* \| \le \frac{\gamma + \lambda \gamma}{1 - \lambda \gamma} \| Q_k - Q^* \|,
  \end{equation*}
  and for $\lambda < \frac{1 - \gamma}{2 \gamma}$ the sequence $(Q_k)_{k \geq 1}$ converges to $Q^*$ exponentially fast.
\end{lemma}

\begin{proof}
Fix $\mu$ and let $B = (I-\lambda\gamma P^{\mu})^{-1}$. Using \eqref{eq:RlambdaQ_Tstar}, we write
\begin{align*}
\R^*_\lambda Q - Q^* &= B \left [ \T Q - Q + (I - \lambda \gamma P^\mu) (Q - Q^*) \right ] \\
&= B \left [ \T Q - Q^* - \lambda \gamma P^\mu (Q - Q^*) \right ].
\end{align*}
Taking the sup-norm, since $\| \T Q - Q^* \| \le \gamma \| Q - Q^* \|$,
we deduce the result:
\begin{equation*}
\big \| \R^*_\lambda Q - Q^* \big \| \le \frac{\gamma + \lambda \gamma}{1-\lambda \gamma} \big \| Q - Q^* \big \|.
\end{equation*}
\end{proof}

\subsection{Online Convergence}
\label{sec:online-convergence}
We are now ready to prove the online convergence of Algorithm 1. Let
the following hold for every sample trajectory
$\tau_k$ and all $x \in \X, a \in \A$:
\begin{enumerate}
\item \textbf{Minimum visit frequency: } $\sum_{t\geq 0}\P\{ x_t, a_t = x, a\} \ge
  D > 0$.
\item \textbf{Finite trajectories: } $\E_{\mu_k} T_k^2 < \infty$, where $T_k$
  is the length of $\tau_k$.
\item \textbf{Bounded stepsizes: } $\sum_{k\geq 0} \alpha_k(x,a) =
  \infty$, $\sum_{k \geq 0} \alpha_k^2(x,a) < \infty$.
\end{enumerate}
Assumption 2 requires trajectories to be finite w.p. 1, which is
satisfied by {\em proper} behavior policies. 
Equivalently, we may require from the MDP that all trajectories eventually reach a zero-value absorbing
state.
The proof closely follows that of Proposition 5.2 from
\cite{bertsekas1996neurodynamic}, and requires rewriting the update in the suitable
form, and verifying Assumptions (a) through (d) from their Proposition
4.5.

\begin{proof} (Sketch)
Let $z_{k,t}(x,a) \defequal \sum_{s = 0}^t (\gamma \lambda)^{t-s}
\indic{(x_s, a_s) = (x, a)}$ denote the accumulating trace. It follows
from Assumptions 1 and 2 that the total update at phase $k$ is
bounded, which allows us to write the online version of~\eqref{eqn:online_q_pi_update_every_visit} as 
\begin{align*}
  Q^{o}_{k+1}(x,a) & \gets (1 - D_k\alpha_k) Q^{o}_k(x,a) + D_k \alpha_k \big(
                 \R^{\pi_k,\mu_k}_\lambda Q^{o}_k(x, a) + w_k +
                     u_k \big) \\
  w_k & \defequal (D_k)^{-1}\Big[ \sum_{t\geq 0} z_{k,t}
\delta^{\pi_k}_{t} - \E_{\mu_k} \big[ \sum_{t\geq 0} z_{k,t}
  \delta^{\pi_k}_{t} \big]\Big], \\
  u_k & \defequal  (D_k\alpha_k)^{-1} \big( Q^o_{k+1}(x, a) -
              Q_{k+1}(x, a) \big),
\end{align*}
where $D_k(x, a) \defequal \sum_{t \ge 0} \Pr\{ x_t, a_t
= x, a\}$, and we use the shorthand $y_k \equiv y_k(x, a)$ for $\alpha_k$,
$D_k$, $w_k$, $u_k$, and $z_{k, t}$.
Combining Assumptions 1 and 2, we have $0 < D \le D_k(x,a) < \infty$,
which, combined in turn with Assumption 3, assures that the new
stepsize sequence $\tilde\alpha_k(x, a) = (D_k\alpha_k)(x, a)$
satisfies Assumption (a) of Prop. 4.5.  Assumptions (b) and (d) require
the variance of the noise term $w_k(x, a)$ to be bounded, and the
residual $u_k(x, a)$ to converge to zero, both of which can be shown identically to the corresponding
results from \cite{bertsekas1996neurodynamic}, if Assumption 2 
and Assumption (a) are satisfied. Finally, Assumption (c) is satisfied by Lemmas
\ref{lem:contraction_for_q_pi} and \ref{lem:contraction_for_q_star}
for the policy evaluation and control cases,
respectively.\footnote{Note that the control case goes through without
  modifications, for the values of $\lambda$ prescribed by
  Lemma~\ref{lem:contraction_for_q_star}.} We conclude that the sequence $(Q^o_k)_{k \in \bN}$
converges to $Q^\pi$ or $Q^*$ in the respective settings, w.p. 1.
\end{proof}

\section{Experimental Results}
\label{sec:experiments}

Although we do not have a proof of the $\lambda-\epsilon$ tradeoff
(see 
Section \ref{sec:algorithm}) in the control case
, we wished to investigate whether such a
tradeoff can be observed experimentally. To this end, we applied Q$^*$($\lambda$) to
the Bicycle domain~\citep{randlov98learning}. Here, the agent must
simultaneously balance a 
bicycle and drive it to a goal position. 
Six real-valued variables describe the state -- angle, velocity,
 etc. -- of the bicycle.
The reward function is proportional
to the angle to the goal, and gives -1 for falling and +1 for reaching
the goal. 
The discount factor is 0.99. The Q-function
was approximated using multilinear interpolation over a uniform grid
 of size $10 \times \dots \times 10$, and the stepsize was tuned to 0.1.
We are chiefly interested in the interplay between the $\lambda$ 
parameter in Q$^*$($\lambda$) and an $\epsilon$-greedy exploration policy. 
Our main performance indicator is the frequency at which the goal is reached
by the greedy policy after 500,000 episodes of
training.
We report three findings:
\begin{enumerate}
    \item{Higher values of $\lambda$ lead to improved learning;}
    \item{Very low values of $\epsilon$ exhibit lower performance; and} 
    \item{The Q-function diverges when $\lambda$ is high relative to $\epsilon$.}
\end{enumerate}
Together, these findings suggest that there is indeed a
$\lambda-\epsilon$ tradeoff in the control case as well, and lead us
to conclude that with proper care it can be beneficial to do
off-policy control with Q$^*$($\lambda$).

\begin{figure*}[ht!]
\begin{center}
\includegraphics[width=2.1in]{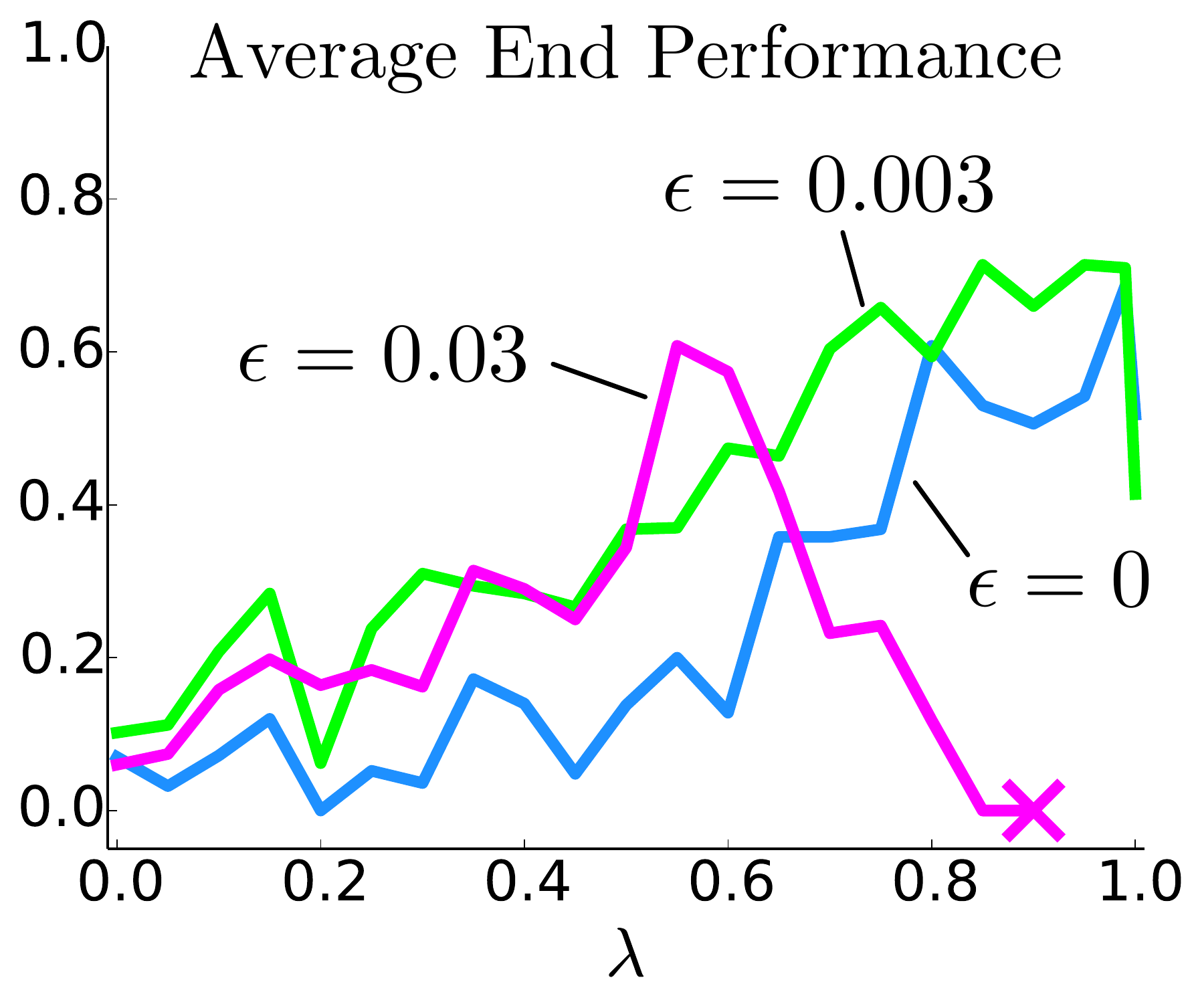}
\hspace{4em}
\includegraphics[width=2.1in]{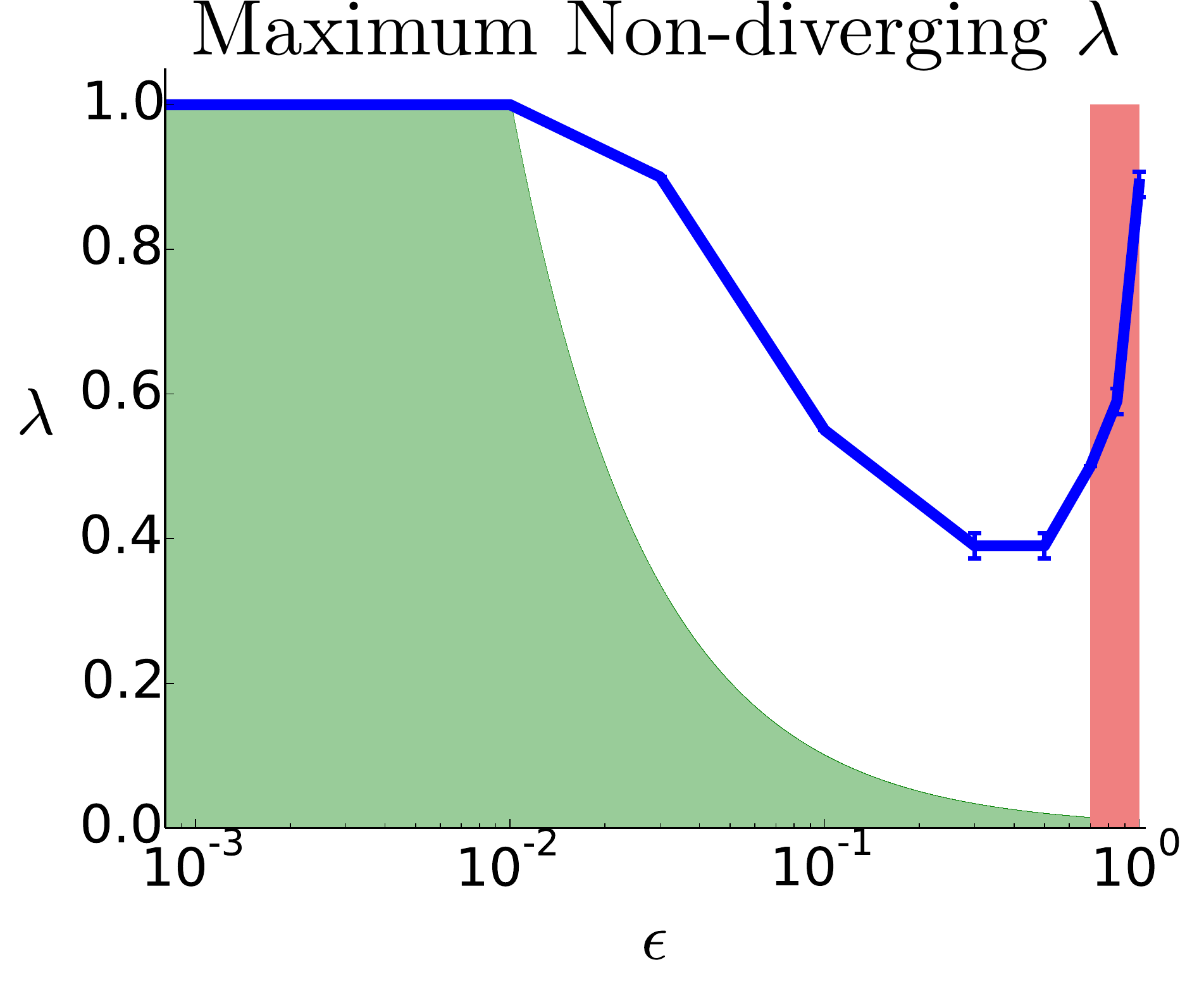}
\caption{{\textbf{Left.} Performance of Q$^*$($\lambda$) on the Bicycle
  domain. Each configuration is an average of five trials. 
  The 'X' marks the lowest value of $\lambda$ for which $\epsilon=0.03$ causes divergence.
\textbf{Right.} Maximum non-diverging $\lambda$ in function of
$\epsilon$. The left-hand shaded region corresponds to our
hypothesized bound. Parameter settings in the right-hand shaded region
do not produce meaningful policies.}}
\label{fig:excellent_bicycle_results}
\end{center}
\end{figure*}

\noindent \textbf{Learning speed and performance.} Figure
\ref{fig:excellent_bicycle_results} (left) depicts the performance of
Q$^*$($\lambda$), in terms of the goal-reaching frequency, for
three values of $\epsilon$. The agent performs best ($p < 0.05$) for
$\epsilon \in [0.003, 0.03]$ and high (w.r.t. $\epsilon$) values of $\lambda$.\footnote{Recall that Randl{\o}v and Alstr{\o}m's agent was trained using SARSA($\lambda)$ with $\lambda = 0.95$.}

\noindent \textbf{Divergence.} For each value of $\epsilon$, we
determined the highest {\em safe} choice of $\lambda$ which did not result in
divergence. As Figure \ref{fig:excellent_bicycle_results} (right)
illustrates, there is a marked decrease in what is a safe value of
$\lambda$ as $\epsilon$ increases. Note the left-hand shaded region
corresponding to the \emph{policy evaluation} bound $\frac{1-\gamma}{\gamma \epsilon}$. Supporting our hypothesis on the true bound on
$\lambda$ (Section \ref{sec:analysis}), it appears clear that the
maximum safe value of $\lambda$ depends on $\epsilon$. In particular,
notice how $\lambda = 1$ stops diverging exactly where predicted by
this bound. 

\section{Related Work}\label{sec:related_work}

In this section, we place the presented algorithms in context of the
existing work in TD($\lambda$)~\cite{sutton-barto98}, focusing in
particular on action-value methods. 
As usual, let $(x_t, a_t, r_t)_{t\geq 0}$ be a trajectory generated by
 following a behavior policy $\mu$, i.e.~$a_t\sim \mu(\cdot|x_t)$. At
 time $s$, SARSA($\lambda$)~\cite{rummery1994line} updates its $Q$-function as
 follows:
 \begin{align}
 Q_{s+1}(x_{s},a_{s}) & \gets Q_s(x_s,a_s) + \alpha_s(\underbrace{A^\lambda R^{(n)}_s -
   Q(x_s,a_s)}_{\Delta_s}), \label{eq:genericQupdate} \\
  R^{(n)}_s & =\sum_{t=s}^{s+n} \gamma^{t-s} r_t +
 \gamma^{n+1}Q(x_{s+n+1},a_{s+n+1}), \label{eq:Rn_Q}
 \end{align}
where $\Delta_s$ denotes the update made at time $s$, and can be
rewritten in terms of one-step TD-errors:

\begin{align}
  \Delta_s & =\sum_{t\geq s}
  (\lambda\gamma)^{t-s}\delta_t, \label{eq:Delta_q} \\
  \delta_t & = r_t + \gamma{Q(x_{t+1},a_{t+1})} - Q(x_t,a_t). \notag 
\end{align}
SARSA($\lambda$) is an on-policy algorithm 
and converges to the value
function  $Q^\mu$ of the behavior policy.
Different algorithms arise by instantiating $R^{(n)}_s$ or $\Delta_s$
from~\eqref{eq:genericQupdate} differently.
Table~\ref{tab:summary} provides the full details, 
while in text we will specify the most revealing
components of the update.

\subsection{Policy Evaluation}
One can imagine considering {\em expectations} over
action-values at the corresponding states  $\E_\pi{Q(x_t,\cdot)}$, in place of the value of
the sampled action $Q(x_t,a_t)$, i.e.:
\begin{align}
  \delta_t & = r_t + \gamma\E_\pi{Q(x_{t+1},\cdot)} -
  \E_\pi{Q(x_{t},\cdot)}. \label{eq:delta_expected_q} 
\end{align}
This is the one-step update for {\em General
  Q-Learning}~\cite{VanHasselt:2011}, which is a generalization of {\em
  Expected SARSA}~\cite{van2009theoretical} to arbitrary
policies. We refer to the direct eligibility trace
extensions of these algorithms formed via Equations~\eqref{eq:genericQupdate}-\eqref{eq:Delta_q} by General Q($\lambda$) and Expected
SARSA($\lambda$) (first mentioned by \citet{sutton2014new}) 
Unfortunately, in an off-policy setting, General Q($\lambda$)
will not converge to 
the value function $Q^\pi$ of the target
policy, as stated by the following proposition.
\begin{proposition} {} \label{prop:offpol_expected_sarsa} 
The stable point of General Q($\lambda$) is $ Q^{\mu,\pi} =
  (I-\lambda\gamma(P^{\mu}-P^{\pi}) - \gamma P^{\pi})^{-1} r$ which
  is the fixed point of the operator $(1-\lambda)\T^\pi + \lambda\T^\mu$.
\end{proposition}
\begin{proof}
Writing the algorithm in operator form, we get
\begin{align*}
  \R Q &=(1-\lambda) \sum_{n\geq 0} \lambda^n \Big[ \sum_{t=0}^{n}
  \gamma^t (P^{\mu})^{t} r + \gamma^{n+1} (P^{\mu})^{n} P^{\pi} Q
  \Big] \\
& =  \sum_{t\geq 0} (\lambda\gamma)^t (P^{\mu})^t \Big[ r +
(1-\lambda)\gamma P^{\pi} Q \Big]
= (I-\lambda\gamma P^{\mu})^{-1} \Big[ r + (1-\lambda)\gamma P^{\pi}
Q \Big].
\end{align*} 
Thus the fixed
point $Q^{\mu,\pi}$ of $\R$ satisfies the following:
\begin{align*}
  Q^{\mu,\pi} & = (I-\lambda\gamma P^{\mu})^{-1} \Big[ r +
                (1-\lambda)\gamma P^{\pi}Q^{\mu,\pi} \Big] 
= (1-\lambda) \T^{\pi} Q^{\mu,\pi} + \lambda \T^{\mu} Q^{\mu,\pi}.
\end{align*}
Solving for $Q^{\mu,\pi}$ yields the result.
\end{proof}

Alternatively to replacing both terms with an expectation, one may
only replace the value at the {\em next} state
$x_{t+1}$ by $\E_\pi{Q(x_{t+1},\cdot)}$,
obtaining:
\begin{align}
  \delta^\pi_t & = r_t + \gamma\E_\pi{Q(x_{t+1},\cdot)} -
  Q(x_{t},a_t). \label{eq:delta_pi}
\end{align}
This is exactly our policy evaluation algorithm Q$^\pi$($\lambda$).
Specifically, when $\pi = \mu$, we get
the on-policy Q$^\pi$($\lambda$). The induced {\em on-policy}
correction may serve as a variance reduction term for Expected
SARSA($\lambda$) (it may be helpful to refer to the $n$-step return in
Table~\ref{tab:summary} to observe this), but we leave variance
analysis of this algorithm for future work.
When $\pi\neq\mu$, we recover off-policy
Q$^\pi$($\lambda$), which (under the stated conditions) converges to $Q^\pi$.

\subsubsection{Target Policy Probability Methods:}
The algorithms above directly descend from basic SARSA($\lambda$), but
often learning off-policy requires special treatment. For example, a typical off-policy technique 
is importance sampling (IS)~\cite{Precup2001}. It is a classical
Monte Carlo method
that allows one to sample from the available distribution,
but obtain (unbiased or consistent) samples of the desired one, by
reweighing the samples with their likelihood ratio according to the two
distributions. That is, the updates for the ordinary {\em
  per-decision} IS algorithm for policy evaluation are made as follows:
\begin{align*}
   \Delta_s 
   &= \sum_{t\geq s}
   (\lambda\gamma)^{t-s}\delta_t\prod_{i=s+1}^{t}\frac{\pi(a_i|x_i)}{\mu(a_i|x_i)}\\ 
 \delta_t 
   &= r_t +
   \gamma\frac{\pi(a_{t+1}|s_{t+1})}{\mu(a_{t+1}|s_{t+1})}Q(x_{t+1},a_{t+1})
   - Q(x_t,a_t). 
\end{align*}
This family of algorithms converges to $Q^\pi$ with probability $1$, under any soft,
stationary behavior $\mu$~
\cite{precup2000eligibility}.  
There are several (recent) off-policy algorithms that reduce the variance of IS methods, at the cost of
added bias~\cite{mahmood2015off,mahmood2015emphatic,1509.05172}.

However, off-policy Q$^\pi$($\lambda$) is perhaps related closest
to the {\em Tree-Backup (TB) algorithm}, also discussed by
\citet{precup2000eligibility}. Its one-step TD-error is the same
as \eqref{eq:delta_pi}, the algorithms back up the same tree, and
neither requires knowledge of the behavior policy $\mu$. The
important difference is in the weighting of the updates. As an off-policy
precaution, TB($\lambda$) weighs updates along a trajectory with the
cumulative target probability of that trajectory up until that point:
\begin{align}
  \Delta_s = \sum_{t\geq s} (\lambda\gamma)^{t-s}\delta^\pi_t\prod_{i=s+1}^{t}\pi(a_i|x_i).\label{eq:Delta_TB}
\end{align}

The weighting
simplifies the convergence argument, allowing TB($\lambda$) to converge to $Q^\pi$ without
further restrictions on the distance between $\mu$ and $\pi$~\cite{precup2000eligibility}.  The drawback of TB($\lambda$) is that in the case of
near on-policy-ness (when $\mu$ is close to $\pi$) the product of the
probabilities cuts the traces unnecessarily (especially
when the policies are stochastic). What we show in this paper, is that
plain TD-learning {\em can} converge off-policy with no
special treatment, subject to a tradeoff condition on $\lambda$ and
$\epsilon$. Under that condition, Q$^\pi$($\lambda$) applies
both on- and off-policy, without modifications. An ideal algorithm
should be able to automatically cut the traces (like TB($\lambda$)) in
case of extreme off-policy-ness while reverting to Q$^\pi$($\lambda$)
when being near on-policy.

\subsection{Control}
\label{sec:related-work-control}
Perhaps the most popular version of Q($\lambda$) is due to
\citet{watkins1992}.
Off-policy, it truncates the return and bootstraps as soon as the
behavior policy takes a non-greedy action, as described by the following update:
\begin{align}
  \Delta_s & =
  \sum_{t=s}^{s+\tau}(\lambda\gamma)^{t-s}\delta_t, \label{eq:watkins-q-lambda} 
\end{align}
where $\tau=\min\{u\geq 1: a_{s+u}\notin \argmax_a Q(x_{s+u}, a)\}$.
Note that this update is a special case of \eqref{eq:Delta_TB} for
deterministic greedy policies, with $\prod_{i={s+1}}^t\indic{a_i \in \argmax_a Q(x_i, a)}$ replacing the probability
product.
When the policies $\mu$
and $\pi$ are not too similar, and $\lambda$ is not too
 small, the truncation may greatly reduce the benefit of complex
 backups.

Q($\lambda$) of \citet{peng1996incremental} is meant to remedy this, by
being a hybrid between SARSA($\lambda$) and Watkins's
Q($\lambda$). Its $n$-step return $\sum_{t=s}^{s+n} \gamma^{t-s} r_t + \gamma^{n+1} \max_a Q(x_{s+n+1}, a)$
requires the following form for the TD-error:
 \begin{align}
   \delta_t & = r(x_t, a_t) + \gamma \max_a
 Q(x_{t+1}, a) - \max_a Q(x_t, a). \notag 
\end{align}
This is, in fact, the same update rule as the General Q($\lambda$) defined in \eqref{eq:delta_expected_q},
where $\pi$ is the greedy policy. Following the same steps as in the proof of Proposition~\ref{prop:offpol_expected_sarsa}, the limit of this algorithm (if it converges) will be the fixed point of the operator $(1-\lambda)\T + \lambda\T^\mu$ which is different from $Q^*$ unless the behavior is always greedy. 

\citet{sutton-barto98} mention another, {\em naive} version of
Watkins's Q($\lambda$) that does not cut the trace on non-greedy actions. That
is exactly the Q$^*$($\lambda$) algorithm described in
this paper. Notice that despite the similarity to Watkins's Q($\lambda$), the equivalence
representation for Q$^*$($\lambda$) is different from the one that would
be derived by setting $\tau=\infty$ in \eqref{eq:watkins-q-lambda},
since the $n$-step return uses the {\em corrected} immediate reward
$r_t + \gamma \max_a Q(x_{t}, a) - Q(x_t,a_t)$ 
instead of the immediate reward alone. This correction is invisible in Watkins's
Q($\lambda$), since the behavior policy is assumed to be greedy,
before the return is cut off.

\section{Conclusion}
\label{sec:closing}

We formulated new algorithms of the TD$(\lambda)$ family for
off-policy policy evaluation and control. Unlike traditional off-policy learning algorithms, these methods
do not involve weighting returns by their policy probabilities, yet
under the right conditions converge to the correct TD fixed points.
In policy evaluation, convergence is subject to a tradeoff between the
degree of bootstrapping $\lambda$, distance between policies
$\epsilon$, and the discount factor $\gamma$.
In control, determining the existence of a non-trivial $\epsilon$-dependent bound for $\lambda$ remains an open problem. Supported by telling empirical
results in the Bicycle domain, we hypothesize that such a bound
exists, and closely resembles the $\frac{1-\gamma}{\gamma \epsilon}$
bound from the policy evaluation case.

\section*{Acknowledgements}
The authors thank Hado van Hasselt and others at Google DeepMind, as well as the anonymous reviewers for
their thoughtful feedback on the paper.

\small
\bibliographystyle{plainnat}

\begin{landscape}
\begin{table}[h]

\centering
\caption{Comparison of the update rules of several learning algorithms
  using the $\lambda$-return. We show both the $n$-step return and the
  resulting update rule for the $\lambda$-return from any state $x_s$
  when following a behavior policy $a_t\sim \mu(\cdot|x_t)$. \textbf{Top part.
  Policy evaluation algorithms:} 
  SARSA($\lambda$), Expected
  SARSA($\lambda$), General
  Q($\lambda$), Per-Decision Importance Sampling (PDIS($\lambda$)),
  TB($\lambda$), and Q$^{\pi}$($\lambda$), in both on-policy
  (i.e.~$\pi=\mu$) and off-policy settings (with a target policy
  $\pi\neq \mu$). Note the same $Q^{\pi}(\lambda)$ equation
  applies to both on- and off-policy settings. We
  abbreviate $\pi_i \equiv \pi(a_i|x_i)$, $\mu_i \equiv \mu(a_i|x_i)$,
  $\rho_i \equiv
  \pi_i/\mu_i$, and write $\E^{a\neq b}_\pi Q(x,\cdot) \equiv
  \sum_{a\in\A\backslash b} \pi(a|x) Q(x,a)$.
    \textbf{Bottom part, control algorithms:} Watkins's $Q(\lambda)$,
Peng and Williams's $Q(\lambda)$, and $Q^*(\lambda)$. The \textbf{FP} column
denotes the stable point of these algorithms (i.e. the fixed point of the expected
update), regardless of whether the algorithm converges to it.
General Q($\lambda$) may converge to $Q^{\mu,\pi}$ defined as the fixed
point of the Bellman operator $(1-\lambda)\T^{\pi}+\lambda \T^{\mu}$.
The fixed point of Watkins's Q($\lambda$) is $Q^*$ but the case $\lambda>0$
may not be significantly better than $\lambda=0$ (regular Q-learning) 
if the behavior policy is different from the greedy one. The fixed point $Q^{\mu,*}$ of Peng and Williams's $Q(\lambda)$ is the fixed point of
$(1-\lambda)\T+\lambda \T^{\mu}$, which is different from $Q^*$ when
$\mu\neq\pi$ (see Proposition 1).
The algorithms analyzed in this paper are $Q^{\pi}(\lambda)$ and $Q^*(\lambda)$, for which convergence to respectively $Q^\pi$ and $Q^*$ occurs under some conditions (see Lemmas). 
}
\begin{tabular}{l||l|l|l} 
{\bf Algorithm}  & {\bf $n$-step return} & {\bf Update rule for the
  $\lambda$-return} & {\bf FP} \\
\hline\hline
 TD($\lambda$) & $\sum_{t=s}^{s+n} \gamma^{t-s} r_t + \gamma^{n+1} V(x_{s+n+1})$ &  $\sum_{t\geq s} (\lambda\gamma)^{t-s} \delta_t$ & $V^{\mu}$ \\
 (on-policy)        &       &  $\delta_t = r_t+\gamma V(x_{t+1}) - V(x_t)$ &   \\ 
 \hline
SARSA($\lambda$) & $\sum_{t=s}^{s+n} \gamma^{t-s} r_t + \gamma^{n+1} Q(x_{s+n+1}, a_{s+n+1})$  & $\sum_{t\geq s} (\lambda\gamma)^{t-s} \delta_t$ & $Q^{\mu}$ \\
(on-policy)        &       &  $\delta_t = r_t+\gamma Q(x_{t+1}, a_{t+1}) - Q(x_t, a_t)$ &   \\ 
\hline
$\E$ SARSA($\lambda$)& $\sum_{t=s}^{s+n} \gamma^{t-s} r_t + \gamma^{n+1} \E_{\mu} Q(x_{s+n+1}, \cdot)$  & $\sum_{t\geq s} (\lambda\gamma)^{t-s} \delta_t + \E_{\mu} Q(x_s,\cdot)-Q(x_s, a_s)$ & $Q^{\mu}$ \\
(on-policy) &       &  $\delta_t = r_t+\gamma \E_{\mu} Q(x_{t+1}, \cdot) - \E_{\mu} Q(x_t, \cdot)$ &   \\ 
\hline
General Q($\lambda$)& $\sum_{t=s}^{s+n} \gamma^{t-s} r_t + \gamma^{n+1} \E_{\pi} Q(x_{s+n+1}, \cdot)$  & $\sum_{t\geq s} (\lambda\gamma)^{t-s} \delta_t + \E_{\pi} Q(x_s,\cdot)-Q(x_s, a_s)$ & $Q^{\mu,\pi}$ \\
(off-policy) &       &  $\delta_t = r_t+\gamma \E_{\pi} Q(x_{t+1}, \cdot) - \E_{\pi} Q(x_t, \cdot)$ &   \\ 
\hline
PDIS($\lambda$) &  $\sum_{t=s}^{s+n}
\gamma^{t-s} r_t\prod_{i=s+1}^{t}\rho_i$
&  $\sum_{t\geq s}
(\lambda\gamma)^{t-s}\delta_t\prod_{i=s+1}^{t}\rho_i$
& $Q^{\pi}$ \\
(off-policy) &  +$\gamma^{n+1} Q(x_{s+n+1}, a_{s+n+1})\prod_{i=s}^{s+n}\rho_i $ &  $\delta_t = r_t
+\gamma\rho_{t+1}Q(x_{t+1},a_{t+1})
- Q(x_t,a_t)$ 
&  \\ 
 \hline
 TB($\lambda$) & $\sum_{t=s}^{s+n} \gamma^{t-s}\prod_{i=s+1}^{t}\pi_i \big[ r_t + \gamma\E^{a\neq a_{t+1}}_{\pi} Q(x_{t+1}, \cdot)\big]$
 & $\sum_{t\geq s} (\lambda\gamma)^{t-s} \delta_t\prod_{i=s+1}^{t}\pi_i$ & $Q^{\pi}$ \\
 (off-policy) &  $+\gamma^{n+1} \prod_{i=s+1}^{s+n+1}\pi_i Q(x_{s+n+1}, a_{s+n+1})$ 
 &  $\delta_t = r_t+\gamma \E_{\pi} Q(x_{t+1}, \cdot) - Q(x_t, a_t)$ &  \\ 
\hline
\textcolor{blue}{{\bf $\mathbf Q^\pi(\lambda)$}} & $\sum_{t=s}^{s+n} \gamma^{t-s} \big[ r_t +  \E_{\pi} Q(x_{t}, \cdot) -  Q(x_{t}, a_{t}) \big]$  & $\sum_{t\geq s} (\lambda\gamma)^{t-s} \delta_t $ & $Q^{\pi}$ \\
(on/off-policy) & $+\gamma^{n+1} \E_{\pi} Q(x_{s+n+1}, \cdot)$      &  $\delta_t = r_t+\gamma \E_{\pi} Q(x_{t+1}, \cdot) - Q(x_t, a_t)$ &   \\ 
\hline\hline
Q($\lambda$) & $\sum_{t=s}^{s+n} \gamma^{t-s} r_t + \gamma^{n+1}
\max_a Q(x_{s+n+1}, a)$  & $\sum^{s+\tau}_{t = s} (\lambda\gamma)^{t-s}
\delta_t\prod_{i=s+1}^t $ & $Q^*$ \\
(Watkins's) &   (for any $n < \tau = \argmin_{u \geq 1}\indic{\pi_{s+u}\neq\mu_{s+u}}$)
&  $\delta_t = r_t+\gamma \max_a Q(x_{t+1}, a) - Q(x_t, a_t)$ &   \\ [1.2ex]
\hline
Q($\lambda$) & $\sum_{t=s}^{s+n} \gamma^{t-s} r_t + \gamma^{n+1} \max_a Q(x_{s+n+1}, a)$  & $\sum_{t=s}^{s+n} (\lambda\gamma)^{t-s} \delta_t + \max_a Q(x_s,a)-Q(x_s, a_s) $ 
                    & $Q^{\mu,*}$ \\
(P \& W's) &    &  $\delta_t = r_t+\gamma \max_a Q(x_{t+1}, a) - \max_a Q(x_t, a)$ 
                    &   \\ 
\hline
\textcolor{blue}{{\bf $\mathbf Q^*(\lambda)$}} & $\sum_{t=s}^{s+n} \gamma^{t-s} \big[ r_t +  \max_a Q(x_{t}, a) -  Q(x_{t}, a_{t}) \big]$ 
                                         & $\sum_{t\geq s} (\lambda\gamma)^{t-s} \delta_t $ & $Q^{*}$ \\
       & $+\gamma^{n+1} \max_a Q(x_{s+n+1}, a)$      &  $\delta_t = r_t+\gamma \max_a Q(x_{t+1}, a) - Q(x_t, a_t)$ &   \\ 
\hline\hline
\end{tabular}
\label{tab:summary}
\end{table}
\end{landscape} 

\end{document}